\documentclass[11pt]{article}
\usepackage[utf8]{inputenc}
\usepackage{cmap}           
\usepackage[T1]{fontenc}
\usepackage{lmodern}
\usepackage[protrusion=true,expansion=false]{microtype}
\usepackage[margin=1in, top=1.1in, bottom=1.1in]{geometry}
\usepackage{amsmath,amssymb,amsthm}
\usepackage{booktabs}
\usepackage{graphicx}
\usepackage[colorlinks=true, linkcolor=blue!70!black, citecolor=blue!70!black, urlcolor=blue!70!black]{hyperref}
\usepackage{xcolor}
\usepackage{enumitem}
\usepackage[numbers,sort&compress]{natbib}
\usepackage{caption}
\usepackage{subcaption}
\usepackage{tikz}
\usetikzlibrary{shapes.geometric,arrows.meta,positioning,fit,backgrounds,calc}
\usepackage{listings}
\usepackage{float}
\usepackage{array}
\usepackage{multirow}

\newtheorem{theorem}{Theorem}
\newtheorem{proposition}{Proposition}
\newtheorem{definition}{Definition}

\title{\textbf{Operationalizing Reconstructive Authority}\\[4pt]
       \large Runtime Construction, Dependency Resolution, and Execution Gating\\
       in Autonomous Agent Systems}

\author{Marcelo Fernandez\\[2pt]
        \small TraslaIA\\[2pt]
        \small \href{https://agentcontrolprotocol.xyz}{agentcontrolprotocol.xyz}}

\date{}

\begin{document}

\maketitle


\begin{table}[h!]
\centering
\small
\caption*{\textbf{Agent Governance Series}}
\begin{tabular}{@{}llll@{}}
\toprule
\textbf{Paper} & \textbf{Title} & \textbf{Zenodo DOI} & \textbf{arXiv} \\
\midrule
P0 & Atomic Decision Boundaries~\cite{fernandez2026a}           & 10.5281/zenodo.19670649 & 2604.17511 \\
P1 & Agent Control Protocol (ACP)~\cite{fernandez2026b}         & 10.5281/zenodo.19672575 & 2603.18829 \\
P2 & From Admission to Invariants (IML)~\cite{fernandez2026c}   & 10.5281/zenodo.19672589 & 2604.17517 \\
P3/4 & Irreducible Governance Structure~\cite{fernandez2026govstr} & 10.5281/zenodo.19708496 & TBD \\
P5 & Reconstructive Authority Model (RAM)~\cite{fernandez2026ram} & 10.5281/zenodo.19669430 & TBD \\
P6 & \textbf{Operationalizing Reconstructive Authority}          & 10.5281/zenodo.19699460 & TBD \\
\bottomrule
\end{tabular}
\end{table}


\begin{abstract}
Autonomous agent systems fail not only due to incorrect decisions, but due to executing decisions whose authority no longer holds at runtime. Prior work defined Reconstructive Authority (RAM) as a condition for valid execution: actions are permitted only if authority can be constructed from current state.

This paper addresses enforcement at runtime: how to enforce this condition in a running system.

We introduce a runtime execution model in which authority is evaluated at action time and execution is conditioned on its constructibility. This extends the execution state space beyond admit/deny with a third state, \textit{halt}, representing cases where authority is undefined due to incomplete or uncertain observability.

We define a concrete execution protocol including dynamic dependency resolution, authority reconstruction, and explicit decision semantics. We further introduce a Recovery Loop that integrates drift detection (IML) with execution control (ACP), allowing the system to suspend execution, acquire missing information, and re-attempt authority reconstruction.

We show that this model guarantees safety—no action is executed without constructible authority—and conditional liveness: execution resumes when authority-defining variables become observable.

This work operationalizes reconstructive authority as a runtime enforcement mechanism, providing the execution semantics required to apply RAM in real systems.
\end{abstract}

\tableofcontents
\newpage

\section{Introduction}

Autonomous systems operate in environments where the conditions under which decisions are
made do not remain stable until execution. This gap between decision and execution is not
incidental; it is structural. Prior work has shown that separating evaluation from execution
introduces inconsistencies~\cite{fernandez2026a,lamport1978time}, that enforcement alone cannot detect
behavioral drift~\cite{fernandez2026c}, and that maintaining coherent behavior under limited
observability requires multiple interacting layers~\cite{fernandez2026govstr}.

A central failure mode follows from this structure: systems execute decisions whose
authority no longer holds at runtime. This is not a failure of decision quality, but of
execution validity.

\textit{Example.} An agent is authorized to execute a financial transfer at time $t_0$,
when the recipient account status, transaction limits, and contextual risk indicators are
all verified. By time $t_1$, the recipient account has been flagged and the risk profile
has changed. The admission-time authority computed at $t_0$ no longer holds, but absent
runtime reconstruction the system executes regardless---this is execution under stale
authority, a structural failure that validation-only approaches cannot prevent.

Reconstructive Authority (RAM), introduced in Paper~5~\cite{fernandez2026ram}, addresses
this by redefining execution as a constructibility problem: authority must be derived from
the current system state, not inherited from prior validation. However, while RAM
establishes the theoretical condition for valid execution, it does not define how such a
condition is enforced in a running system.

\textit{This paper provides that operationalization. This paper does not extend the theoretical definition of Reconstructive Authority. Instead, it addresses a distinct problem: how to enforce authority reconstruction as a runtime execution protocol.}

We introduce a runtime execution model in which authority is evaluated at the point of
action, and execution is conditioned on its constructibility. This leads to a fundamental
extension of the execution state space: beyond traditional admit/deny outcomes, we define
a third state, \textsc{halt}, representing situations where authority cannot be established
due to incomplete or uncertain information.

The introduction of \textsc{halt} changes the nature of execution control. Systems are no
longer forced to approximate decisions under uncertainty; instead, they can explicitly
refuse to act when authority is undefined.

We further show that authority is not a continuous function of system state. Small
perturbations can invalidate execution, and sequences of decisions accumulate divergence
over time. Without runtime reconstruction, these properties guarantee eventual execution
under invalid authority.

To address this, we integrate RAM within the Agent Control Protocol
(ACP)~\cite{fernandez2026b} as the authorization criterion at the execution gate (stage~3
of the 6-stage ACP flow), and connect it with the Invariant Monitoring Layer
(IML)~\cite{fernandez2026c} for drift detection. This enables a closed-loop execution
model. ACP remains the enforcement pipeline; RAM refines its authorization semantics.

The key contribution of this paper is the introduction of a Recovery Loop that governs
system behavior after \textsc{halt}. Rather than treating \textsc{halt} as a terminal
state, the system identifies missing authority-defining variables, acquires additional
information, and attempts reconstruction. We prove that, under recoverable observability
conditions, the system will eventually resume execution.

Taken together, these elements define a complete execution model with two guarantees:

\begin{itemize}
  \item \textbf{Safety:} no action is executed without constructible authority.
  \item \textbf{Liveness:} execution resumes when authority becomes constructible.
\end{itemize}

This reframes execution in autonomous systems from a problem of estimation and validation
to one of constructibility and enforcement at runtime.

\section{System Model}

Let:
\begin{itemize}
  \item $S_r(t)$: real system state at time $t$
  \item $O(t) \subseteq S_r(t)$: observable state
  \item $A(t)$: execution authority
  \item $F$: authority construction function
  \item $U: V \to [0,1]$: uncertainty map over state variables; $U(v)$ quantifies the
        degree of uncertainty for variable $v \in V$. A required variable triggers
        $A(t) = \bot$ when $U(v) > \theta_{\mathrm{auth}}$, where
        $\theta_{\mathrm{auth}}$ is the authority-resolution threshold.
\end{itemize}

Authority is defined as:
\[
  A(t) = F(S_r(t))
\]

Since $S_r(t)$ is not fully observable, the system operates on $O(t)$ and must determine
whether authority can be constructed.

\begin{definition}[Authority Constructibility]
\label{def:constructible}
Let $a$ be an action and $D(a)$ the set of variables required to evaluate authority
for $a$. Authority is \emph{constructible} in state $O(t)$ if and only if all three
conditions hold:
\begin{enumerate}
  \item \textbf{Observability:} all dependencies $d \in D(a)$ are observable in $O(t)$,
        i.e.\ $U(d) \leq \theta_{\mathrm{auth}}$.
  \item \textbf{Closure:} the dependency graph induced by $D(a)$ is closed under
        resolution (no unresolved upstream dependencies remain).
  \item \textbf{Consistency:} no invariant violations are detected over the evaluated
        dependencies.
\end{enumerate}
Formally:
\[
  \mathrm{Constructible}(a,\,O(t)) \;\triangleq\;
  \mathrm{Observable}(D(a),\,O(t)) \;\land\;
  \mathrm{Closed}(D(a)) \;\land\;
  \mathrm{Consistent}(D(a))
\]
If this condition holds, authority can be evaluated as either valid or invalid.
Otherwise, authority is \emph{non-evaluable}.
\end{definition}

\begin{definition}[Undefined Authority]
If $\mathrm{Constructible}(a,\,O(t))$ does not hold (i.e.\ any required variable is
unobservable, the dependency graph is open, or invariants are violated), then:
\[
  A(t) = \bot
\]
and execution must not proceed.
\end{definition}

\begin{definition}[Execution State]
\label{def:execstate}
Given an action $a$ and system state $O(t)$, the execution outcome is:
\[
  \mathrm{ExecState}(a,\,O(t)) =
  \begin{cases}
    \textsc{execute} & \text{if } \mathrm{Constructible}(a,\,O(t)) \land A(t) = \mathrm{True} \\
    \textsc{deny}    & \text{if } \mathrm{Constructible}(a,\,O(t)) \land A(t) = \mathrm{False} \\
    \textsc{halt}    & \text{if } \lnot\,\mathrm{Constructible}(a,\,O(t))
  \end{cases}
\]
The execution gate therefore admits exactly three outcomes:
$\mathrm{gate}(t) \in \{\textsc{execute},\; \textsc{deny},\; \textsc{halt}\}$.
\end{definition}

\textbf{The \textsc{halt} state is the central innovation of this paper.} It is distinct
from \textsc{deny}, and this distinction is not merely terminological:
\begin{itemize}
  \item \textsc{deny} corresponds to \emph{falsified authority}: authority is constructible
        and evaluates to False ($A(t) = \mathrm{False}$).
  \item \textsc{halt} corresponds to \emph{non-evaluable authority}: authority cannot be
        determined because one or more constructibility conditions fail ($A(t) = \bot$).
\end{itemize}
\textbf{Halt occurs when authority cannot be evaluated due to incomplete dependency
resolution, not when it is invalid.} Systems that collapse this distinction treat
epistemic incompleteness as a denial decision, which is incorrect: a non-evaluable
authority is neither valid nor invalid, and executing under it---in either direction---
violates the safety condition.

\section{Architectural Overview}

Operational RAM consists of four tightly integrated components.

\subsection{State Acquisition Layer}

Collects $O(t)$ from internal system signals, external inputs, and contextual environment.
No completeness assumption is made. Any gap in observability propagates directly to the
authority construction step.

\subsection{Authority Construction Engine}

Evaluates:
\[
  A(t) = F(O(t))
\]
via dynamic dependency resolution. The engine identifies which variables are required,
evaluates their observability and reliability, and returns one of the three authority
states.

\subsection{Execution Gate}

Enforces:
\[
  \begin{cases}
    A(t) = \text{True} & \Rightarrow \textsc{execute} \\
    A(t) = \text{False} & \Rightarrow \textsc{deny} \\
    A(t) = \bot        & \Rightarrow \textsc{halt}
  \end{cases}
\]

No execution occurs without a defined, positive authority state.

\subsection{Audit Layer}

Records for each execution attempt: the observed state $O(t)$, the dependency set $D(t)$,
the authority decision, and failure conditions. This ensures full traceability without
influencing the execution decision.

\section{Runtime Authority Construction Protocol}

We now define the operational protocol that implements RAM in real systems.

\subsection{Inputs}

At time $t$, the protocol receives:
\begin{itemize}
  \item Observable state $O(t)$
  \item Action class $C$
  \item Policy prior $P_0(C)$ (candidate generator only)
  \item Construction function $F$
  \item Uncertainty map $U$
\end{itemize}

\subsection{Prior Influence Constraints}

The policy prior $P_0(C)$ is strictly constrained:

\begin{itemize}
  \item It \textbf{cannot} authorize execution independently.
  \item It \textbf{cannot} override runtime construction.
  \item It \textbf{can only} propose initial candidate variables for $F$ to evaluate.
  \item Runtime construction may promote variables outside the prior when required.
  \item Demotion of required variables is not permitted within the same decision cycle
        under uncertainty.
\end{itemize}

\begin{proposition}[Prior Non-Authority]
No execution decision can be derived from $P_0(C)$ alone. Authority requires runtime
construction over $O(t)$.
\end{proposition}

\subsection{Dependency Resolution Algorithm}
\label{sec:dep-resolution}

Let $D(t)$ be the set of variables required to construct authority. The key property is:
\[
  D(t) = \mathrm{Resolve}(O(t))
\]
Dependencies are \emph{not statically defined}. They emerge from the current observable
state at runtime.

The algorithm proceeds in seven steps:

\begin{enumerate}
  \item Initialize candidate variables from $P_0(C)$.
  \item Evaluate $F$ against current state $O(t)$.
  \item Derive required dependency set $R(t)$ from actual dependencies discovered by $F$.
  \item Promote all required variables into the authority-defining set $A_d(t)$.
  \item For each variable $v \in A_d(t)$: evaluate observability and uncertainty status.
  \item If any authority-defining variable is unresolved $\Rightarrow A(t) = \bot$,
        execution halts.
  \item If all variables are resolved and constraints pass $\Rightarrow A(t) = \text{True}$,
        execution proceeds.
\end{enumerate}

\subsection{Construction Function}

\begin{lstlisting}[language=Python, caption={Runtime authority construction}]
def reconstruct_authority(state, prior):
    candidates = initialize_from_prior(prior)
    required   = resolve_dependencies(state, candidates)

    for var in required:
        if not observable(var, state):
            return UNDEFINED
        if not reliable(var, state):
            return UNDEFINED

    if not satisfies_constraints(state, required):
        return False

    return True
\end{lstlisting}

\subsection{Non-Persistence Property}

Authority is not persistent across time steps:
\[
  A(t) \neq A(t + \Delta)
\]
unless explicitly recomputed from fresh state. This property is not an engineering
limitation but a logical consequence of the reconstructive model: authority depends on
state, and state changes.

\subsection{Worked Example}
\label{sec:worked-example}

Consider an agent with three authority-defining variables:
$x_1$ (account status), $x_2$ (transaction limit), $x_3$ (risk indicator).
Let $\theta_{\mathrm{auth}} = 0.2$.

\begin{itemize}
  \item \textbf{Scenario A — \textsc{execute}:} $x_1$ observed (status: active),
        $x_2$ observed (within bounds), $x_3$ observed with $U(x_3) = 0.1 <
        \theta_{\mathrm{auth}}$. All dependencies resolved, constraints satisfied.
        $A(t) = \mathrm{True}$ $\Rightarrow$ \texttt{ADMIT\_AUTHORITY\_CONSTRUCTIBLE}.

  \item \textbf{Scenario B — \textsc{deny}:} $x_1$, $x_2$, $x_3$ all observable,
        but the transaction exceeds the limit encoded in $x_2$; constraints fail.
        $A(t) = \mathrm{False}$ $\Rightarrow$ \texttt{DENY}.

  \item \textbf{Scenario C — \textsc{halt}:} $x_3$ is observable but
        $U(x_3) = 0.35 > \theta_{\mathrm{auth}}$. The authority-defining uncertainty
        for $x_3$ cannot be resolved. $A(t) = \bot$ $\Rightarrow$
        \texttt{HALT\_AUTHORITY\_UNDEFINED\_UNCERTAINTY}; the Recovery Loop is invoked.
\end{itemize}

The same action class, three different outcomes---determined entirely by the
observable state and uncertainty levels at the moment of execution.

\subsection{Runtime Enforcement Loop}
\label{sec:enforcement-loop}

We define the enforcement loop that integrates RAM, IML, and ACP into a unified
runtime execution process. This loop is what makes authority constructibility a
continuously enforced property, not a one-time check.

\begin{itemize}
  \item \textbf{RAM} (this paper): defines authority and dependency resolution.
  \item \textbf{IML} (Paper~2~\cite{fernandez2026c}): detects drift and invariant violations.
  \item \textbf{ACP} (Paper~1~\cite{fernandez2026b}): enforces execution gating and transitions.
\end{itemize}

\begin{lstlisting}[language=Python, caption={Runtime enforcement loop integrating RAM, IML, and ACP}]
function enforce(action a):
  while True:
    S_t = observe_system_state()        # State Acquisition Layer

    # Step 1: Dependency resolution (RAM)
    D = resolve_dependencies(a, S_t)

    # Step 2: Constructibility check (RAM -- Definition 1)
    if (not is_observable(D, S_t)
        or not is_closed(D)
        or not is_consistent(D)):
      state = HALT
      wait_for_update()                 # Recovery Loop (Section 9)
      continue

    # Step 3: Authority evaluation (RAM)
    if is_valid(a, D, S_t):
      state = EXECUTE
    else:
      state = DENY

    # Step 4: Enforcement (ACP -- stage 3 of 6)
    if state == DENY:
      reject_execution(a)
      return DENY

    if state == EXECUTE:
      # Step 5: Pre-execution revalidation (IML)
      if drift_detected(S_t) or invariants_violated(D, S_t):
        state = HALT
        continue                        # Re-enter loop

      # Step 6: Execute
      execute(a)

      # Step 7: Post-execution monitoring (IML)
      if drift_detected(S_t):
        trigger_recovery(a)

      return EXECUTE
\end{lstlisting}

\noindent\textbf{Key properties of the loop:}
\begin{itemize}
  \item \textit{Safety}: Execution is only reachable through the \textsc{execute} state,
        which requires both constructibility and validity by construction.
  \item \textit{Drift handling}: Any detected inconsistency at Step~5 forces a transition
        to \textsc{halt}, preventing execution under stale conditions.
  \item \textit{Recovery}: The loop re-enters evaluation until authority becomes
        constructible again (see Section~\ref{sec:recovery}).
  \item \textit{Separation of concerns}: RAM defines what must be true; IML checks whether
        it remains true; ACP decides whether execution proceeds.
\end{itemize}

\section{Drift Handling}

Drift is modeled as change in $S_r(t)$ that invalidates previously held authority
components.

RAM handles drift implicitly: any change that affects $O(t)$ requires recomputation of
$A(t)$. If reconstruction fails under the new state, execution halts.

\textbf{IML integration.} Recomputation is triggered by IML anomaly signals and is
mandatory before every state-mutating action. Continuous polling is not required. This
event-driven model makes IML and RAM orthogonal and complementary: IML observes
deviations~\cite{fernandez2026c}; RAM enforces the execution consequence.

\begin{proposition}[Drift Manifestation]
Drift in $S_r(t)$ manifests as failure of authority construction. No explicit drift
detector is required for RAM to provide safety; however, IML-triggered recomputation
improves liveness by reducing unnecessary halts.
\end{proposition}

\medskip
\noindent\textbf{Drift timeline.}
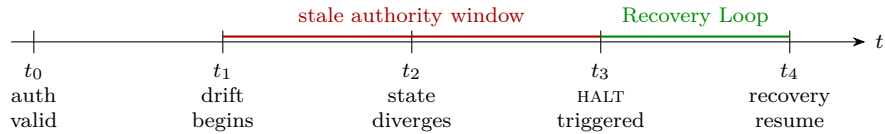
\begin{center}
\begin{tikzpicture}[>=Stealth, font=\scriptsize]
  \draw[->] (-0.3,0) -- (11,0) node[right]{$t$};
  \draw (0,   0.15) -- (0,   -0.15) node[below, align=center]{$t_0$\\auth\\valid};
  \draw (2.5, 0.15) -- (2.5, -0.15) node[below, align=center]{$t_1$\\drift\\begins};
  \draw (5,   0.15) -- (5,   -0.15) node[below, align=center]{$t_2$\\state\\diverges};
  \draw (7.5, 0.15) -- (7.5, -0.15) node[below, align=center]{$t_3$\\\textsc{halt}\\triggered};
  \draw (10,  0.15) -- (10,  -0.15) node[below, align=center]{$t_4$\\recovery\\resume};
  \draw[red!65!black, thick] (2.5, 0.07) -- (7.5, 0.07)
        node[midway, above, text=red!65!black]{stale authority window};
  \draw[green!55!black, thick] (7.5, 0.07) -- (10, 0.07)
        node[midway, above, text=green!55!black]{Recovery Loop};
\end{tikzpicture}
\end{center}
\medskip

\begin{theorem}[Stale Authority Under Drift]
Any system that does not recompute authority at execution time will eventually execute
under invalid authority when $S_r(t)$ changes continuously.
\end{theorem}

\textit{Intuition.} Authority derived at admission time reflects a snapshot of state
conditions. As state changes continuously, that snapshot expires. Without recomputation,
the system acts on authority that no longer corresponds to real conditions---an
arbitrarily large gap between the authority-constructing state and the execution-time
state accumulates over time.

\begin{proof}
Let $A(t_0) = \text{True}$ be authority computed at admission time. Let $S_r(t)$ change
continuously such that for some $t_1 > t_0$, the required variables $D(t_0)$ are no
longer valid under $S_r(t_1)$. Since authority is not recomputed, the system executes with
$A(t_1) = \text{True}$ (stale), while the correct value is $A(t_1) = \bot$ or
$A(t_1) = \text{False}$. \hfill $\blacksquare$
\end{proof}

\section{Decision Model}

Each execution attempt produces a deterministic decision code. Table~\ref{tab:codes}
lists the seven codes derived from the operational protocol, along with the triggering
condition and the recommended system action.

\begin{table}[H]
\centering
\caption{Decision codes produced by the Runtime Authority Construction Protocol.
         Every execution attempt resolves to exactly one code.}
\label{tab:codes}
\resizebox{\linewidth}{!}{%
\begin{tabular}{@{}lp{5cm}p{4.5cm}@{}}
\toprule
\textbf{Decision Code} & \textbf{Condition} & \textbf{System Action} \\
\midrule
\texttt{ADMIT\_AUTHORITY\_CONSTRUCTIBLE}
  & All required variables resolved; constraints satisfied
  & Execute \\[4pt]
\texttt{HALT\_AUTHORITY\_UNDEFINED\_REQUIRED\_DEPENDENCY}
  & A dependency required by $F$ is not observable
  & Halt; invoke Recovery Loop \\[4pt]
\texttt{HALT\_AUTHORITY\_UNDEFINED\_UNCERTAINTY}
  & Required variable observable but $U(v) > \theta_{\mathrm{auth}}$
  & Halt; invoke Recovery Loop \\[4pt]
\texttt{HALT\_MISSING\_REQUIRED\_SIGNAL}
  & Signal required to compute $D(t)$ is absent
  & Halt; request signal acquisition \\[4pt]
\texttt{HALT\_REATTESTATION\_REQUIRED}
  & Authority previously valid; state change mandates reconstruction
  & Halt; recompute before proceeding \\[4pt]
\texttt{CONTINUE\_BOUNDED\_NON\_AUTHORITY\_DRIFT}
  & Drift detected but affects only non-authority-defining variables
  & Execute; log drift event \\[4pt]
\texttt{NARROW\_PRIVILEGE\_REEVALUATE}
  & Authority partially constructible
  & Reduce action scope; re-evaluate \\
\bottomrule
\end{tabular}%
}
\end{table}

\begin{proposition}[Execution Safety Invariant]
If $A(t) \neq \text{True}$, no action in execution class $C$ proceeds. This is enforced
structurally by the execution gate, not by policy configuration.
\end{proposition}

\section{Auditability and Traceability}

For each execution attempt, the system produces a structured audit artifact:

\begin{itemize}
  \item Timestamp and action identifier.
  \item Initial prior candidates from $P_0(C)$.
  \item Runtime required set $R(t)$ with discovery path.
  \item Promotions into $A_d(t)$ with reason.
  \item Uncertainty status $U(v)$ for each authority-defining variable.
  \item Final decision code and halt or continue rationale.
\end{itemize}

The audit artifact explains not only the decision, but why specific variables became
authority-defining and how they affected the outcome. This enables post-hoc analysis,
regulatory compliance, and system validation without influencing execution behavior.

\textbf{High-frequency considerations.} In systems where execution attempts occur at high
rates, continuous full-artifact storage may be impractical. Sub-sampling is admissible
provided that: (a)~every \textsc{halt} event and its decision code are logged
unconditionally; (b)~all boundary transitions (\textsc{execute}~$\leftrightarrow$~\textsc{halt})
are always captured; and (c)~sampled artifacts are sufficient to reconstruct the authority
lineage for any queried point in time. Compression may be applied to the $R(t)$ discovery
path while preserving the final decision code and per-variable uncertainty status
$U(v)$.

\section{Integration with the Governance Stack}

RAM does not replace any existing governance layer. It operates as the authorization
criterion within the execution gate of ACP (stage~3 of the 6-stage ACP enforcement
flow). Each layer retains its distinct role:

\subsection{Atomic Execution (Paper~0)}

Paper~0~\cite{fernandez2026a} establishes that decision and action must be atomically
bound to eliminate TOCTOU vulnerabilities. RAM operates within this atomicity guarantee:
the authority construction and the execution it authorizes are part of the same atomic
unit.

\subsection{Enforcement Layer (ACP, Paper~1)}

ACP~\cite{fernandez2026b} provides the enforcement pipeline: capability scoping,
constraint validation, and execution control across 6 stages. RAM replaces the
admission-time authorization criterion at stage~3 with a runtime constructibility check.
ACP remains the pipeline; RAM refines its authorization semantics without altering the
pipeline structure.

\subsection{Behavioral Layer (IML, Paper~2)}

IML~\cite{fernandez2026c} detects behavioral drift and non-identifiability. In the
integrated model, IML anomaly signals trigger RAM recomputation. This event-driven
coupling ensures that authority is reconstructed when the governance stack observes
meaningful state changes, without requiring continuous polling.

\subsection{Allocation and Compositional Governance (Papers~3/4)}

Paper~3/4~\cite{fernandez2026govstr} addresses two interlinked structural questions:
which agents are permitted to reach the execution gate (allocation), and whether the
four governance dimensions are irreducible (compositional structure). RAM operates
orthogonally to both: allocation controls who arrives at the gate; RAM controls what
executes once they arrive. The irreducibility result further establishes that RAM
cannot substitute for any structural layer---runtime constructibility is a necessary
complement to the governance stack, not a replacement for it.

\subsection{Reconstructive Authority Model (Paper~5)}

Paper~5~\cite{fernandez2026ram} defines the formal model that this paper operationalizes.
The relationship is the same as between GAT (Paper~0) and ACP (Paper~1): the theory
defines what must hold; the operationalization defines how it is enforced.

\section{Recovery Loop under Reconstructive Authority}
\label{sec:recovery}

The introduction of the \textsc{halt} state ($A(t) = \bot$) ensures that execution does
not proceed under undefined authority. However, safe systems must also provide a mechanism
to recover from \textsc{halt} when the underlying cause is resolvable.

We define the Recovery Loop as the closed process that transitions the system from
\textsc{halt} to a state where authority may be reconstructible again.

\subsection{HALT Semantics}

A \textsc{halt} occurs when:
\[
  A(t) = F(S_r(t)) = \bot
\]

This may result from:
\begin{itemize}
  \item Missing required signals
  \item Authority-defining uncertainty exceeding resolution bounds
  \item Drift invalidating prior authority assumptions
\end{itemize}

\textsc{halt} is not a terminal failure, but a controlled suspension of execution.

\subsection{Recovery Loop Definition}

Let $t$ be the time of \textsc{halt}. The system enters a recovery process:
\[
  S_r(t) \rightarrow S_r(t') \rightarrow S_r(t'')
\]
through the following stages:

\begin{enumerate}
  \item \textbf{Signal Extraction.}
        Identify the minimal set of unresolved variables $U(t)$ responsible for
        $A(t) = \bot$.

  \item \textbf{IML Trigger.}
        The IML layer~\cite{fernandez2026c} detects drift or observability gaps and emits
        a recovery trigger, directing the state acquisition layer to focus on the
        identified gap.

  \item \textbf{State Augmentation.}
        Acquire additional observations or reduce action scope:
        \[
          S_r(t') = S_r(t) \cup \Delta O
        \]
        If state augmentation cannot resolve the unobservable variables, the system may
        revert to the last atomically consistent state as established by
        Paper~0~\cite{fernandez2026a}, preserving execution integrity without requiring
        complete observability recovery.

  \item \textbf{Reconstruction Attempt.}
        Recompute authority from augmented state:
        \[
          A(t') = F(S_r(t'))
        \]

  \item \textbf{Resolution.}
        \begin{itemize}
          \item If $A(t') \neq \bot$ $\Rightarrow$ execution resumes
          \item If $A(t') = \bot$ $\Rightarrow$ remain in \textsc{halt} or escalate
        \end{itemize}
\end{enumerate}

\subsection{Integration with ACP}

The Recovery Loop operates within the execution gate of ACP:

\begin{itemize}
  \item ACP provides the execution pipeline and enforcement structure.
  \item RAM defines the authorization criterion at the gate.
  \item IML provides drift detection and recovery triggers.
\end{itemize}

Thus, RAM does not replace ACP, but refines its authorization semantics. The Recovery Loop
closes the loop between all three layers: IML $\rightarrow$ RAM $\rightarrow$ ACP
$\rightarrow$ IML.

\begin{theorem}[Conditional Liveness under Recoverable Observability]
\label{thm:liveness}
If all authority-defining variables required by $F$ are eventually observable or
resolvable, then a system implementing the Recovery Loop will eventually exit
\textsc{halt} and resume execution.

\noindent\textbf{Assumptions:}
\begin{enumerate}
  \item \textit{Eventual Observability}: all dependencies $d \in D(a)$ that are
        temporarily unobservable become observable in finite time.
  \item \textit{Finite Resolution}: dependency resolution and consistency checks
        terminate in finite time.
  \item \textit{Stable Evaluation Window}: there exists a non-zero interval during
        which constructibility holds long enough to complete evaluation.
\end{enumerate}

Formally: if $\exists\, t' > t$ such that $\mathrm{Constructible}(a,\,O(t'))$ holds,
then $\exists\, t'' \geq t'$ such that
$\mathrm{ExecState}(a,\,O(t'')) \in \{\textsc{execute},\,\textsc{deny}\}$.
\end{theorem}

\begin{proof}
Assume a \textsc{halt} at time $t$ such that $A(t) = \bot$.
Let $U(t)$ be the set of unresolved authority-defining variables.

By assumption, each variable in $U(t)$ becomes observable or resolvable at some future
time through state augmentation:
\[
  S_r(t) \rightarrow S_r(t') \rightarrow S_r(t'')
\]
such that $U(t'') = \emptyset$.

Then:
\[
  F(S_r(t'')) = \text{True} \;\Rightarrow\; A(t'') \neq \bot
\]

Thus, execution resumes. \textsc{halt} is not terminal when observability gaps are
recoverable; the system guarantees eventual progress. \hfill $\blacksquare$
\end{proof}

\section{Comparative Analysis}

We situate RAM operationalization against existing execution control paradigms.

\subsection{Attestation-Based Systems}

Attestation systems~\cite{costan2016intel,sabt2015trusted} verify integrity of the
execution environment at initialization. They assume completeness: if the initial state
is valid, execution is valid. Under dynamic conditions this assumption fails; authority
derived from initial attestation may not hold at execution time.

\subsection{Risk-Threshold Systems}

Risk-gated systems~\cite{greenblatt2024ai,leike2018ai} bound uncertainty and execute
under ``acceptable risk.'' This conflates bounded uncertainty with constructible authority.
RAM rejects this: partial observability that affects authority-defining variables produces
$A(t) = \bot$ regardless of the magnitude of uncertainty.

\subsection{Reconstructive Systems}

RAM-based systems require construction of authority from current state before every
action. Any unresolved dependency collapses authority to undefined, preventing execution.

\subsection{Traditional Policy Enforcement (XACML, OPA)}

Traditional policy enforcement systems operate under a binary decision model: actions
are either permitted or denied based on policy evaluation over a given state snapshot.

\textbf{XACML (eXtensible Access Control Markup Language)~\cite{oasis2013xacml}}
defines a standardized architecture for attribute-based access control. Decisions are
derived from policy rules evaluated against a request context. XACML assumes that the
attribute context is complete and stable at evaluation time; it does not address the
case where required attributes become unavailable between decision and execution.

\textbf{OPA (Open Policy Agent)~\cite{opa2016}} provides a general-purpose policy
engine widely used in cloud-native infrastructure. Policies written in Rego are
evaluated against input documents to produce allow/deny decisions. While OPA supports
dynamic inputs and partial evaluation, it does not explicitly model non-evaluable states:
if a required data field is absent, the policy returns undefined or defaults to deny,
without distinguishing between denied authority and non-evaluable authority.

\textbf{Policy-as-Code frameworks} (Rego, Sentinel, Kyverno) share the same structural
assumption: all required data is available at evaluation time.

The core distinction of the proposed model is the explicit treatment of
non-evaluability as a first-class execution state. Existing systems either treat missing
data as a default (deny or allow) or resolve uncertainty within the evaluation phase,
without propagating epistemic incompleteness into execution control. This model elevates
non-evaluability to \textsc{halt}, defers execution until authority can be reconstructed,
and couples execution strictly to runtime constructibility.

\textbf{Complementary role.} This model is not intended to replace policy engines such
as XACML or OPA. It operates as a complementary enforcement layer: policy engines define
\textit{what} conditions must hold; RAM ensures those conditions are still valid
\textit{when execution occurs}.

\subsection{AI Control under Intentional Subversion}

Recent work on AI control under intentional subversion~\cite{greenblatt2024ai} addresses
settings where an agent may actively attempt to circumvent oversight mechanisms. RAM
provides a complementary and orthogonal layer. Even in systems already hardened against
subversion, authority must still be constructed from observable state at the point of
action. Under a subversion scenario, authority-defining variables for a contested action
would reflect the adversarially modified state; RAM would identify missing or inconsistent
signals and collapse authority to $\bot$, triggering \textsc{halt} independently of
whether the control layer detected the subversion attempt.

Thus RAM can function as an additional execution gate on top of control-hardened systems:
subversion attempts that alter authority-relevant state are caught at the construction
layer rather than requiring the control layer to anticipate every specific manipulation
strategy.

\begin{theorem}[Authority Validity]
\label{thm:validity}
Any system that permits execution under incomplete authority-defining state cannot
guarantee that no action proceeds without constructible authority.
\end{theorem}

\textit{Intuition.} Permitting execution when authority-defining variables are unobservable
is structurally equivalent to executing under undefined authority. The safety property
cannot be derived probabilistically from partial information; it requires that all required
variables be resolved. The dependency resolution step in Section~\ref{sec:dep-resolution}
closes this gap by construction.

\begin{proof}
Let a system $\mathcal{S}$ permit execution when $D(t) \not\subseteq O(t)$ for some
required set $D(t)$. Then there exists an execution at time $t$ where $A(t) = \bot$ (by
Definition~2) yet $\mathcal{S}$ proceeds. This violates the execution safety invariant
(Proposition~3). RAM satisfies it by construction through the dependency resolution
algorithm (Section~\ref{sec:dep-resolution}). \hfill $\blacksquare$
\end{proof}

\begin{theorem}[Execution Safety under Constructible Authority]
\label{thm:safety}
Under the runtime enforcement loop (Section~\ref{sec:enforcement-loop}), no action is
executed unless its authority is constructible and valid at execution time. Formally:
\[
  \forall\, a,\, O(t):\quad
  \mathrm{Execute}(a,\,O(t)) \;\Rightarrow\;
  \Bigl(\mathrm{Constructible}(a,\,O(t)) \;\land\; A(t) = \mathrm{True}\Bigr)
\]
\end{theorem}

\begin{proof}[Proof Sketch]
\textit{Entry condition.} The enforcement loop evaluates constructibility
(Definition~\ref{def:constructible}) before any execution attempt. If
$\lnot\,\mathrm{Constructible}(a,\,O(t))$, the system transitions to \textsc{halt}
and execution is deferred---no execution path is reachable.

\textit{Authority evaluation.} If constructibility holds, authority is evaluated. If
$A(t) = \mathrm{False}$, the system transitions to \textsc{deny} and execution is
rejected.

\textit{Pre-execution revalidation.} Immediately prior to execution (Step~5 of the
loop), IML checks for drift and invariant violations. Any inconsistency forces a
transition back to \textsc{halt}, preventing execution under stale or degraded state.

\textit{Execution gate (ACP).} Execution is only reachable when the gate is in the
\textsc{execute} state, which requires both constructibility and validity.

\textit{Conclusion.} There exists no path in the enforcement loop that leads to
execution without satisfying both $\mathrm{Constructible}(a,\,O(t))$ and
$A(t) = \mathrm{True}$. Safety is guaranteed by construction. \hfill $\blacksquare$
\end{proof}

\section{Limitations}

While the proposed model strengthens execution safety under dynamic conditions, it
introduces several limitations that must be acknowledged.

\subsection{Dependency on Observability}

The model assumes that required dependencies can eventually be observed and resolved.
In environments with persistent partial observability or unreliable data sources, the
system may remain in the \textsc{halt} state indefinitely.

This shifts the problem from incorrect execution to absence of execution, which may or
may not be acceptable depending on the application domain. The conditional liveness
guarantee (Theorem~\ref{thm:liveness}) explicitly requires eventual observability; it
does not hold if dependencies are permanently unresolvable.

\subsection{Latency Overhead}

Continuous re-evaluation of constructibility and authority introduces runtime overhead.
In high-frequency or low-latency systems, repeated dependency resolution and invariant
checks may impact performance. Optimizations such as caching, incremental evaluation, or
bounded revalidation windows may be required in practice, provided they do not weaken
the constructibility guarantee.

\subsection{Recovery Complexity}

The transition out of \textsc{halt} depends on external conditions (restored
observability, resolved dependencies). Designing robust recovery strategies requires
careful coordination with monitoring and control layers. In poorly instrumented systems,
recovery may be delayed or inconsistent.

\subsection{Specification Burden}

The model relies on explicit definitions of dependency sets $D(a)$, invariants, and
validity conditions. In complex systems, defining and maintaining these elements can be
non-trivial and error-prone. Incomplete or incorrect specifications will produce
incorrect constructibility evaluations.

\subsection{Non-Deterministic Environments}

In highly volatile environments, rapid oscillation between constructible and
non-constructible states may occur. Without stabilization mechanisms, this can lead to
repeated transitions between \textsc{halt} and \textsc{execute}, affecting system
predictability. The Stable Evaluation Window assumption in Theorem~\ref{thm:liveness}
addresses this, but may not hold in adversarial settings.

\subsection{Scope of Guarantees}

The model guarantees safety with respect to execution under constructible authority. It
does not guarantee correctness of the authority model itself, completeness of dependency
specification, or optimality of decisions. These remain external to the enforcement layer.

\section{Trade-offs}

\subsection{Advantages}

\begin{itemize}
  \item \textbf{Safety guarantee:} no action proceeds without constructible authority.
  \item \textbf{Full traceability:} every decision is explained by the audit artifact.
  \item \textbf{Deterministic gating:} outcomes are computable, not probabilistic.
  \item \textbf{Liveness:} Recovery Loop ensures the system resumes when conditions allow.
\end{itemize}

\subsection{Costs}

\begin{itemize}
  \item \textbf{Increased halting frequency:} partial observability produces more
        \textsc{halt} events than threshold-based systems.
  \item \textbf{Higher latency:} recomputation at every execution step adds overhead.
  \item \textbf{Dependence on observability:} liveness is conditional on the eventual
        recoverability of required signals.
\end{itemize}

These costs are the direct consequence of choosing correctness over availability. In
safety-critical systems this trade-off is appropriate: halting is always preferable to
acting without constructible authority.

\section{Conclusion}

This paper introduced a runtime enforcement model that addresses the gap between
decision and execution in dynamic systems. By requiring that authority be constructible
at execution time---in the precise sense of Definition~\ref{def:constructible}---the
model ensures that actions are only performed under conditions that can be actively
validated.

The introduction of a third execution state, \textsc{halt}, allows the system to
explicitly represent and handle cases where authority cannot be evaluated due to
incomplete or unresolved dependencies. This avoids unsafe execution under partial
observability and drift, a limitation inherent to binary enforcement models such as
XACML and OPA, which treat missing data as a default rather than an epistemic signal.

Operationalizing reconstructive authority requires four shifts:

\begin{enumerate}
  \item From validation to \textbf{construction} --- authority must be built from current
        state, not inherited.
  \item From persistence to \textbf{recomputation} --- authority expires with the state
        that produced it.
  \item From risk acceptance to \textbf{execution refusal} --- undefined authority
        mandates \textsc{halt}.
  \item From stopping to \textbf{recovery} --- \textsc{halt} triggers a closed loop that
        attempts to restore constructibility.
\end{enumerate}

We formalized authority constructibility (Definition~\ref{def:constructible}), defined
execution semantics (Definition~\ref{def:execstate}), and presented a runtime enforcement
loop integrating dependency resolution (RAM), invariant monitoring (IML), and execution
control (ACP). Under these conditions, the model guarantees:

\begin{itemize}
  \item \textbf{Safety}: no action proceeds without constructible authority
        (Theorems~\ref{thm:validity} and~\ref{thm:safety}).
  \item \textbf{Liveness}: execution resumes when dependency resolution is eventually
        possible (Theorem~\ref{thm:liveness}).
\end{itemize}

Together, Papers~5 and~6 complete the governance series:

\begin{itemize}
  \item Papers~0--4 define the conditions under which governance is structurally sound.
  \item Paper~5 defines what authority must be.
  \item Paper~6 defines how authority is enforced at runtime.
\end{itemize}

Rather than replacing existing policy systems, this approach extends their guarantees
into runtime execution, providing a structured mechanism to ensure that decisions remain
valid at the moment they are applied. Future work includes optimization strategies for
low-latency environments, formal verification of system properties, and empirical
evaluation in real-world deployments.

\medskip
\noindent\textit{Execution authority is not inherited. It is not validated. It is
constructed. When it cannot be constructed, execution must not occur.}


\newpage

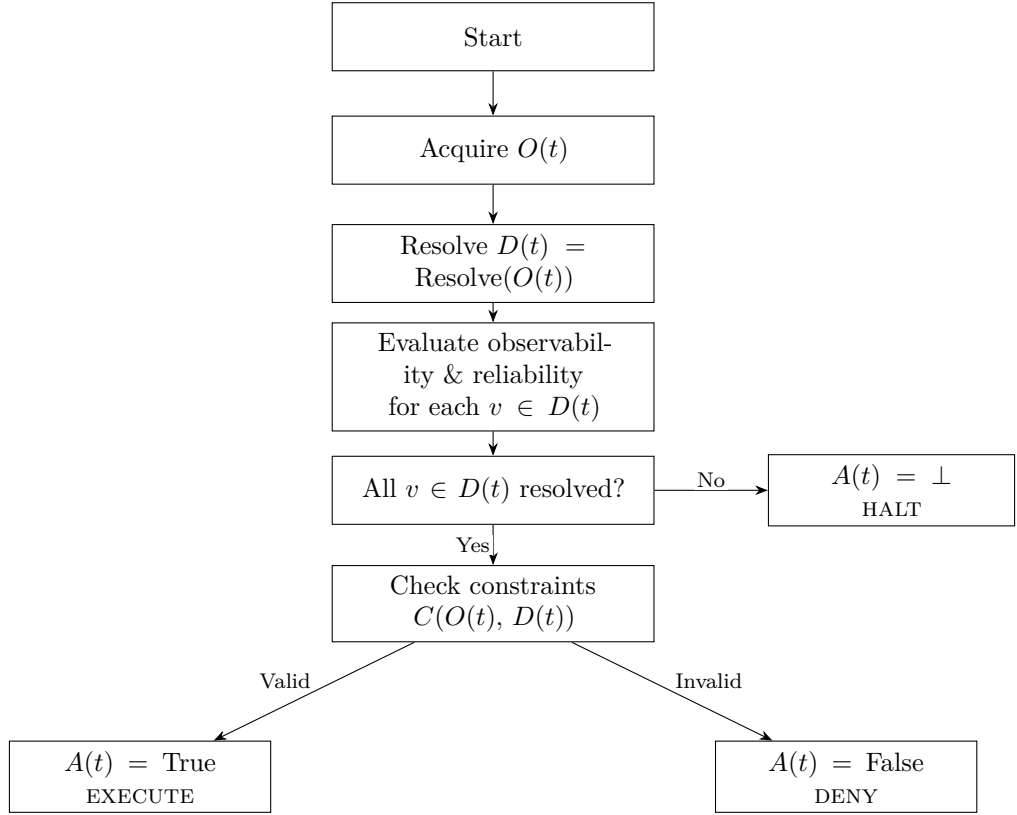
\begin{figure}[H]
\centering
\begin{tikzpicture}[
  >=Stealth,
  node distance=1.5cm,
  box/.style={draw, rectangle, align=center, text width=4cm,
              minimum height=0.9cm, font=\small},
  lbl/.style={draw=none, fill=white, font=\scriptsize, inner sep=1pt}
]
\node[box] (start)       {Start};
\node[box] (obs)         [below of=start]   {Acquire $O(t)$};
\node[box] (dep)         [below of=obs]     {Resolve $D(t) = \mathrm{Resolve}(O(t))$};
\node[box] (eval)        [below of=dep]     {Evaluate observability \& reliability for each $v \in D(t)$};
\node[box] (check)       [below of=eval]    {All $v \in D(t)$ resolved?};
\node[box] (constraints) [below of=check]   {Check constraints $C(O(t),\,D(t))$};

\node[box, text width=3cm] (halt) [right=1.5cm of check]
      {$A(t)=\bot$ \\ \textsc{halt}};

\node[box, text width=3.2cm] (exec)
      [below left=1.3cm and 0.8cm of constraints]
      {$A(t)=\mathrm{True}$ \\ \textsc{execute}};
\node[box, text width=3.2cm] (deny)
      [below right=1.3cm and 0.8cm of constraints]
      {$A(t)=\mathrm{False}$ \\ \textsc{deny}};

\draw[->] (start)       -- (obs);
\draw[->] (obs)         -- (dep);
\draw[->] (dep)         -- (eval);
\draw[->] (eval)        -- (check);
\draw[->] (check)       -- node[lbl, left]{Yes}       (constraints);
\draw[->] (check)       -- node[lbl, above]{No}        (halt);
\draw[->] (constraints) -- node[lbl, above left]{Valid}    (exec);
\draw[->] (constraints) -- node[lbl, above right]{Invalid} (deny);
\end{tikzpicture}
\caption{Runtime authority construction flow. Execution proceeds only if all
         authority-defining dependencies are observable and constraints hold. An
         unresolved dependency immediately yields \textsc{halt}.}
\label{fig:flow}
\end{figure}

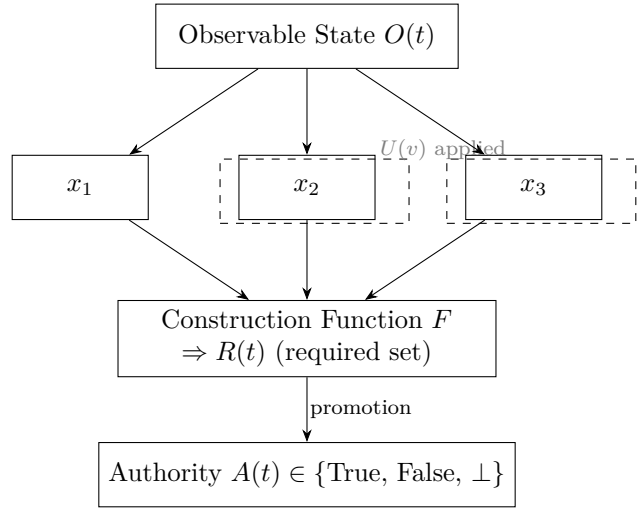
\begin{figure}[H]
\centering
\begin{tikzpicture}[
  >=Stealth,
  box/.style={draw, rectangle, align=center,
              minimum height=0.85cm, font=\small},
  lbl/.style={draw=none, fill=white, font=\scriptsize, inner sep=1pt}
]
\node[box, minimum width=4cm] (state) at (0,0) {Observable State $O(t)$};

\node[box, minimum width=1.8cm] (x1) at (-3,-2)  {$x_1$};
\node[box, minimum width=1.8cm] (x2) at ( 0,-2)  {$x_2$};
\node[box, minimum width=1.8cm] (x3) at ( 3,-2)  {$x_3$};

\node[box, minimum width=2.5cm, dashed] (ux2) at (0.1,-2.05) {};
\node[box, minimum width=2.5cm, dashed] (ux3) at (3.1,-2.05) {};
\node[draw=none, font=\scriptsize, text=gray] at (1.8,-1.5) {$U(v)$ applied};

\node[box, minimum width=5cm] (F) at (0,-4)
      {Construction Function $F$ \\ $\Rightarrow R(t)$ (required set)};

\node[box, minimum width=5cm] (auth) at (0,-5.8)
      {Authority $A(t) \in \{\mathrm{True},\,\mathrm{False},\,\bot\}$};

\draw[->] (state) -- (x1);
\draw[->] (state) -- (x2);
\draw[->] (state) -- (x3);
\draw[->] (x1) -- (F);
\draw[->] (x2) -- (F);
\draw[->] (x3) -- (F);
\draw[->] (F)  -- node[lbl, right]{promotion} (auth);
\end{tikzpicture}
\caption{Dynamic dependency resolution. Observable state $O(t)$ feeds a runtime
         dependency discovery process through $F$. Uncertainty $U(v)$ is evaluated
         per variable. Authority $A(t)$ is constructed, not inherited.}
\label{fig:deps}
\end{figure}

\begin{figure}[H]
\centering
\resizebox{\linewidth}{!}{%
\begin{tikzpicture}[
  >=Stealth,
  node distance=1.8cm,
  box/.style={draw, rectangle, align=center, minimum height=0.9cm, font=\small},
  lbl/.style={draw=none, fill=white, font=\scriptsize, inner sep=1pt},
  arrow/.style={->, thick}
]
\node[box, minimum width=4.5cm] (state)  at (0,0)    {Observable State $O(t)$};
\node[box, minimum width=4.5cm] (uncert) at (5.5,0)  {Uncertainty Map $U(v)$};
\node[box, minimum width=3.5cm] (prior)  at (-6,0)   {Policy Prior $P_0(C)$\\(candidates only)};

\node[box, minimum width=1.8cm] (x1) at (-2,-2)  {$x_1$};
\node[box, minimum width=1.8cm] (x2) at ( 0,-2)  {$x_2$};
\node[box, minimum width=1.8cm] (x3) at ( 2,-2)  {$x_3$};

\node[box, minimum width=5.5cm] (F) at (0,-4)
      {Construction Function $F$\\$\Rightarrow R(t)$ (required set)};

\node[box, minimum width=5.5cm] (authset) at (0,-6)
      {Authority-defining variables $A_d(t)$};

\node[box, minimum width=5.5cm] (gate) at (0,-8)
      {Authority Evaluation\\All $v \in A_d(t)$ observable \& reliable?\\
       \footnotesize$A(t) \in \{\mathrm{True},\,\mathrm{False},\,\bot\}$};

\node[box, minimum width=3.2cm] (exec) at (-4,-10) {$A(t)=\mathrm{True}$\\\textsc{execute}};
\node[box, minimum width=3.2cm] (deny) at ( 0,-10) {$A(t)=\mathrm{False}$\\\textsc{deny}};
\node[box, minimum width=3.2cm] (halt) at ( 4,-10) {$A(t)=\bot$\\\textsc{halt}};

\draw[arrow] (state)   -- (x1);
\draw[arrow] (state)   -- (x2);
\draw[arrow] (state)   -- (x3);
\draw[arrow] (prior)   |- (F);  
\draw[arrow] (x1)      -- (F);
\draw[arrow] (x2)      -- (F);
\draw[arrow] (x3)      -- (F);
\draw[arrow] (uncert)  |- (x3);
\draw[arrow] (uncert.south) -- ++(0,-1) -| (x2);
\draw[arrow] (F)       -- node[lbl, right]{promotion} (authset);
\draw[arrow] (authset) -- (gate);
\draw[arrow] (gate)    -- node[lbl, left]{valid}      (exec);
\draw[arrow] (gate)    -- node[lbl, right]{invalid}   (deny);
\draw[arrow] (gate)    -- node[lbl, right]{unresolved}(halt);
\end{tikzpicture}%
}
\caption{Full RAM protocol diagram. Observable state $O(t)$ and uncertainty $U(v)$ feed
         dynamic dependency resolution through $F$, producing the required set $R(t)$ and
         authority-defining variables $A_d(t)$. Execution proceeds only if all required
         variables are observable and reliable. Policy prior $P_0(C)$ contributes
         candidates only and cannot authorize execution independently.}
\label{fig:fullram}
\end{figure}
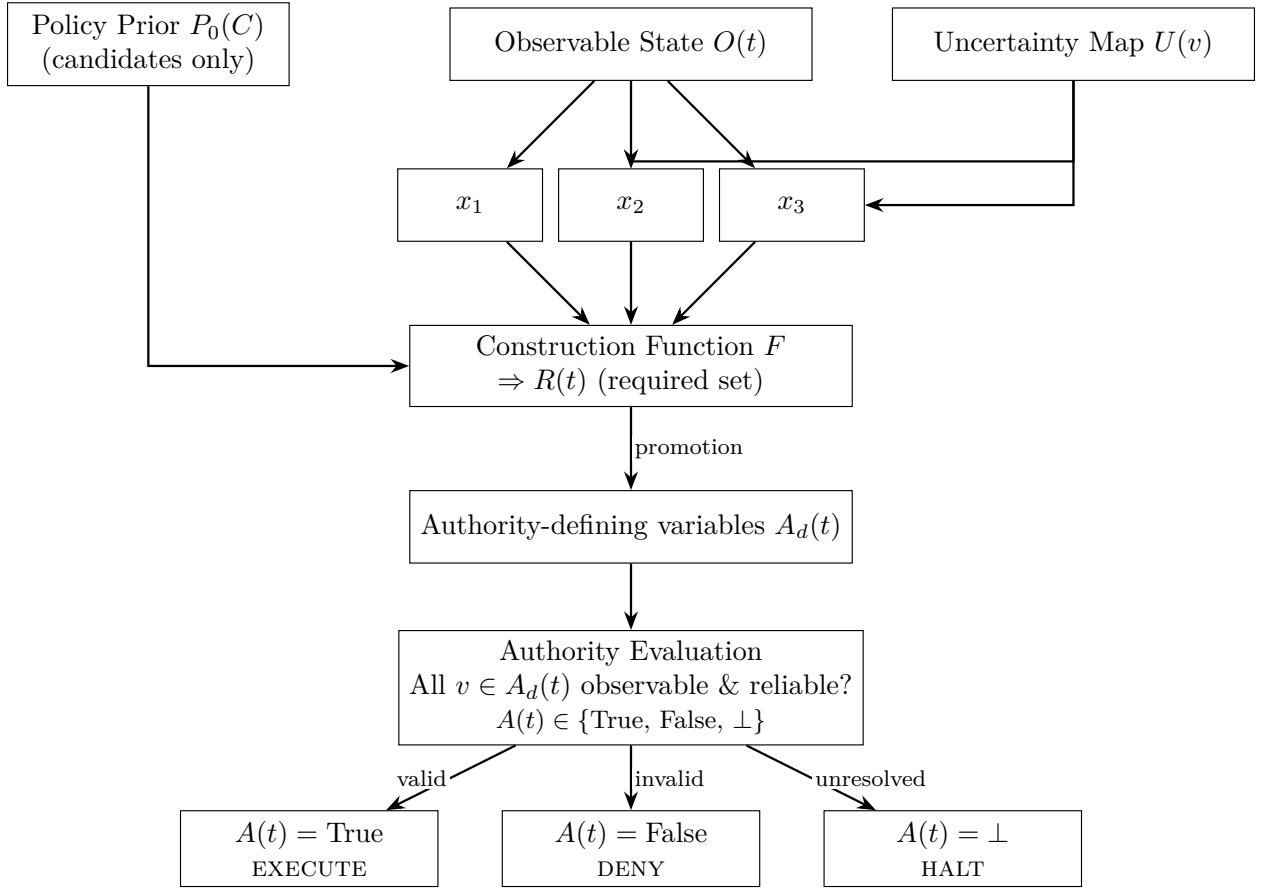

\begin{figure}[H]
\centering
\resizebox{\linewidth}{!}{%
\begin{tikzpicture}[
  >=Stealth,
  node distance=1.2cm,
  layer/.style={draw, rectangle, align=center,
                minimum width=11cm, minimum height=1cm, font=\small},
  arrow/.style={->, thick},
  fback/.style={->, thick, dashed}
]
\node[layer] (alloc)
      {Paper 3 --- Allocation Layer\\(Who is permitted to reach execution?)};
\node[layer] (iml)    [below=of alloc]
      {Paper 2 --- Behavioral Layer (IML)\\(Drift detection, non-identifiability monitoring)};
\node[layer] (acp)    [below=of iml]
      {Paper 1 --- Enforcement Layer (ACP v1.30)\\(Capabilities, constraints, 6-stage execution control)};
\node[layer] (atomic) [below=of acp]
      {Paper 0 --- Atomic Execution Layer\\(Decision-action binding, TOCTOU elimination)};
\node[layer] (comp)   [below=of atomic]
      {Paper 4 --- Compositional Validation\\(P0--P3 are irreducible; governance structure is complete)};
\node[layer] (ram)    [below=of comp]
      {Papers 5--6 --- Reconstructive Authority (RAM)\\Runtime: $A(t)=F(O(t))$ \quad Execute iff $A(t)$ constructible};
\node[layer] (env)    [below=of ram]
      {Environment / Real State $S_r(t)$\\(Partially observable, dynamic, uncertain)};

\draw[arrow] (alloc)  -- (iml);
\draw[arrow] (iml)    -- (acp);
\draw[arrow] (acp)    -- (atomic);
\draw[arrow] (atomic) -- (comp);
\draw[arrow] (comp)   -- (ram);
\draw[arrow] (ram)    -- (env);

\draw[fback] (env.west)   -- ++(-1.5,0) |- (ram.west)
             node[pos=0.25, left, draw=none, font=\scriptsize]{$S_r(t)$};
\draw[fback] (env.west)   -- ++(-2.8,0) |- (iml.west)
             node[pos=0.15, left, draw=none, font=\scriptsize]{drift signal};
\end{tikzpicture}%
}
\caption{Multi-layer governance stack integrating Papers~0--5 with runtime reconstructive
         authority (Paper~6). Allocation controls access, IML detects drift, ACP enforces
         constraints, atomic execution ensures decision binding, and RAM determines
         execution validity at runtime. Dashed arrows show feedback: the environment
         continuously feeds state back, requiring authority reconstruction before every
         action.}
\label{fig:stack}
\end{figure}
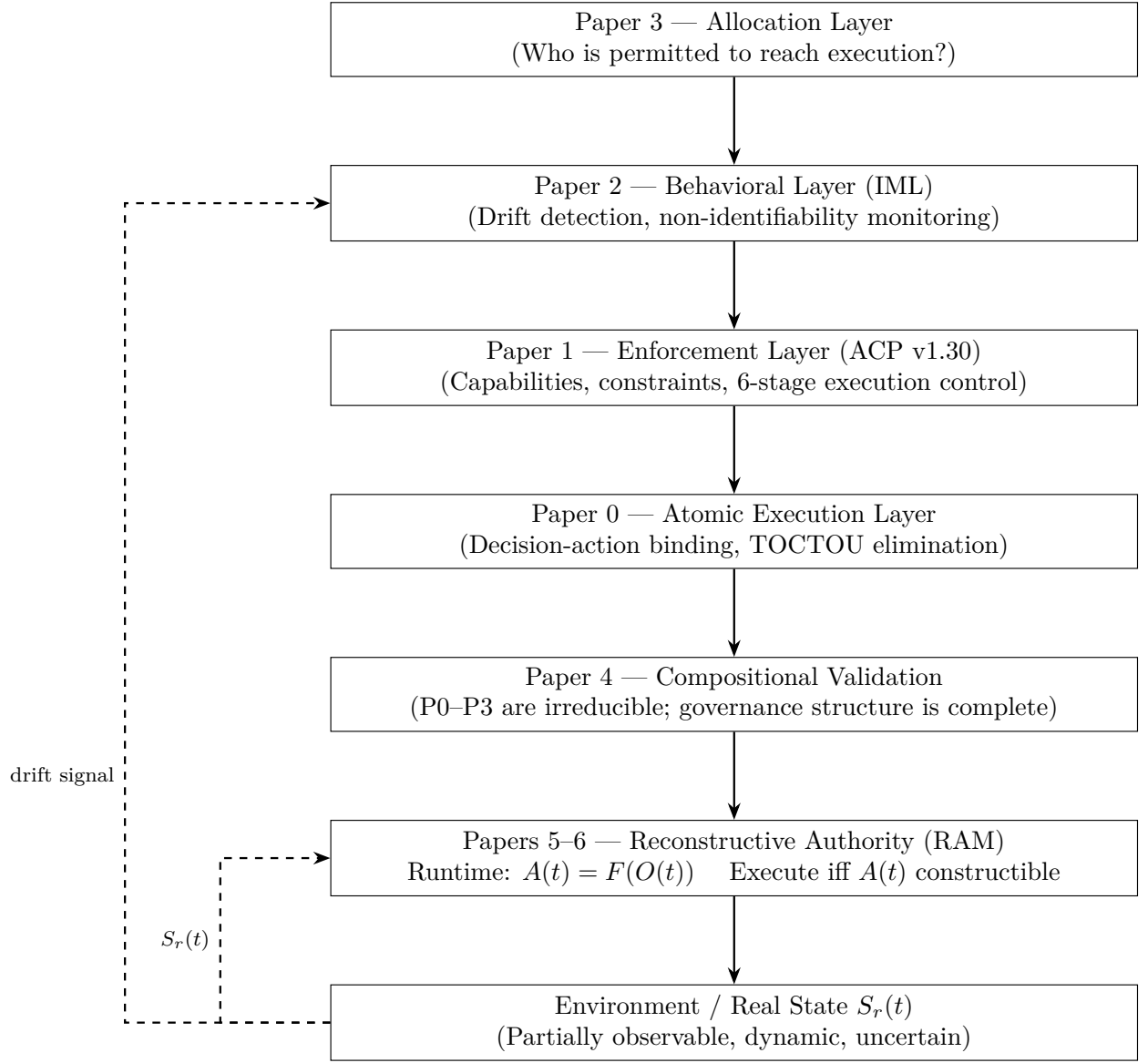

\begin{figure}[H]
\centering
\begin{tikzpicture}[
  >=Stealth,
  node distance=1.5cm,
  box/.style={draw, rectangle, align=center,
              minimum width=4.8cm, minimum height=0.9cm, font=\small},
  lbl/.style={draw=none, fill=white, font=\scriptsize, inner sep=1pt}
]
\node[box] (halt)        {\textsc{halt} \\ $A(t)=\bot$};
\node[box] (signal)      [below of=halt]        {Signal Extraction};
\node[box] (trigger)     [below of=signal]      {IML Trigger};
\node[box] (augment)     [below of=trigger]     {State Augmentation};
\node[box] (reconstruct) [below of=augment]     {Recompute $A(t')$};

\node[box, minimum width=3.5cm] (resume)
      [below left=1.3cm and 1.8cm of reconstruct]  {\textsc{resume}};
\node[box, minimum width=3.5cm] (remain)
      [below right=1.3cm and 1.8cm of reconstruct] {Remain \textsc{halt}};

\draw[->] (halt)        -- (signal);
\draw[->] (signal)      -- (trigger);
\draw[->] (trigger)     -- (augment);
\draw[->] (augment)     -- (reconstruct);
\draw[->, green!60!black, thick] (reconstruct)
      -- node[lbl, above left]{$A(t')\neq\bot$} (resume);
\draw[->, red!65!black, thick] (reconstruct)
      -- node[lbl, above right]{$A(t')=\bot$} (remain);

\draw[->] (remain.east) -- ++(1.4,0) |- (signal.east)
          node[pos=0.3, right, draw=none, font=\scriptsize]{retry};
\end{tikzpicture}
\caption{Recovery Loop under RAM. Upon \textsc{halt}, the system identifies missing
         authority components, triggers re-evaluation via IML, augments state (or
         reverts to the last atomically consistent state per Paper~0), and attempts
         reconstruction. Execution resumes only when $A(t')\neq\bot$; otherwise
         the loop retries.}
\label{fig:recovery}
\end{figure}
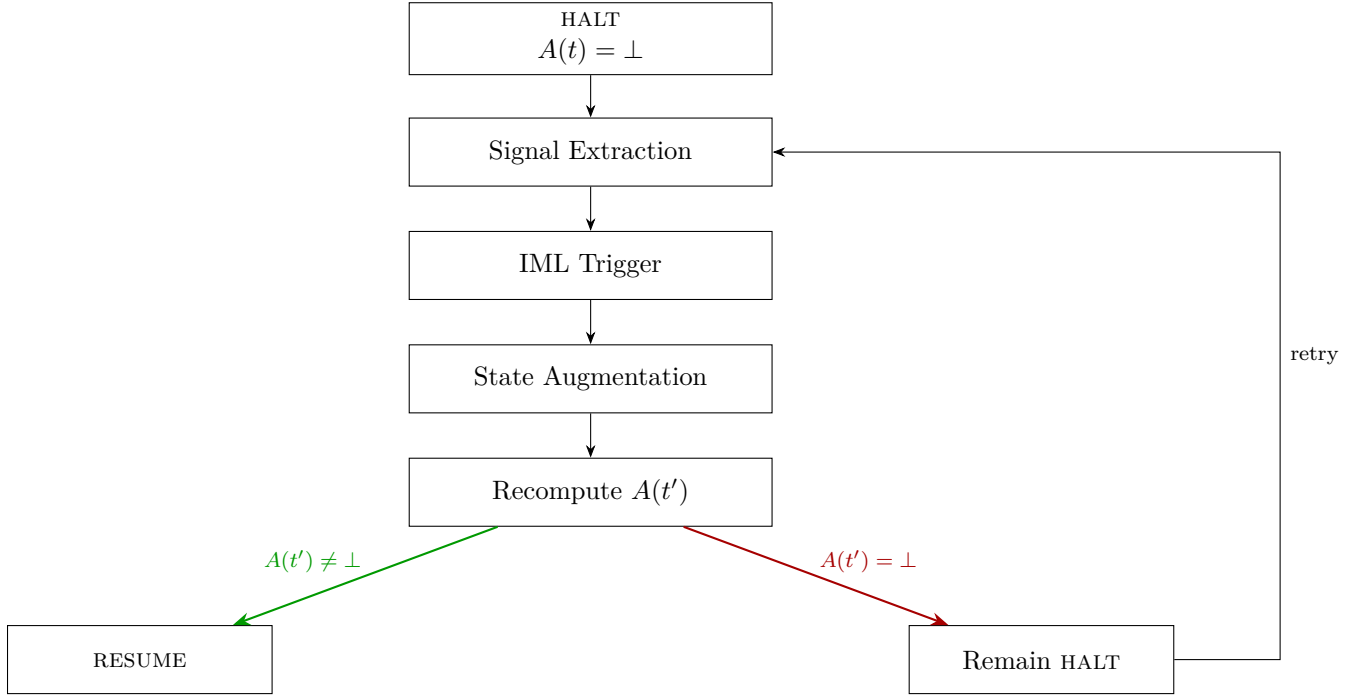

\begin{figure}[H]
\centering
\begin{tikzpicture}[
  >=Stealth,
  node distance=3.8cm,
  state/.style={draw, circle, minimum size=2.4cm, align=center, font=\small,
                thick},
  lbl/.style={draw=none, fill=white, font=\scriptsize, inner sep=2pt,
              align=center, text width=3.2cm}
]
\node[state] (halt)  {\textsc{Halt}\\$A(t)=\bot$};
\node[state] (allow) [right of=halt] {\textsc{Execute}\\$A(t)=\top$};
\node[state] (deny)  [below right=2.5cm and 0.8cm of halt] {\textsc{Deny}\\$A(t)=\bot$};

\draw[->, bend left=20] (halt) to
      node[lbl, above]{Constructible $\land$ Valid} (allow);
\draw[->, bend left=15] (halt) to
      node[lbl, left, text width=2.8cm]{Constructible $\land$ $\lnot$Valid} (deny);
\draw[->, bend left=20] (allow) to
      node[lbl, below]{Drift / Loss of\\Observability} (halt);
\draw[->] (allow) to
      node[lbl, right]{Invariant\\Violation} (deny);
\draw[->, bend left=20] (deny) to
      node[lbl, right, text width=2.4cm]{Re-evaluation\\(deps change)} (halt);

\draw[->] (-2.8, 0) -- node[above, font=\scriptsize]{evaluate} (halt);
\end{tikzpicture}
\caption{Execution state machine under ternary enforcement. Entry is always through
         the evaluation gate. Execution is only permitted in the \textsc{Execute} state,
         which requires both constructibility and validity. Drift or loss of observability
         forces transition back to \textsc{Halt}; the Recovery Loop governs exit from
         \textsc{Halt}.}
\label{fig:statemachine}
\end{figure}
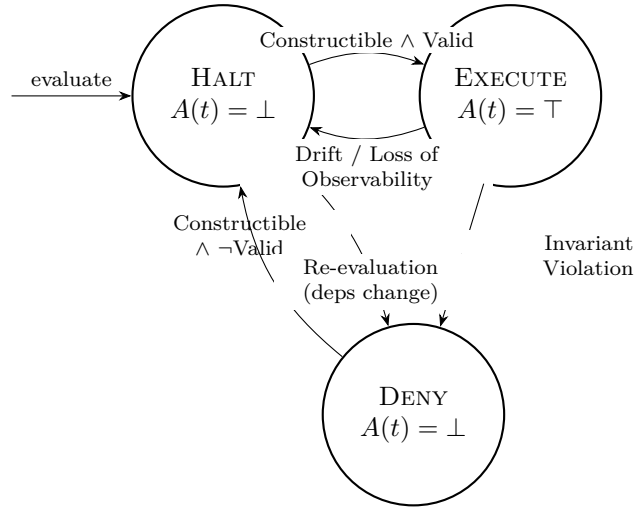


\bibliographystyle{plainnat}
\bibliography{references}

\end{document}